\documentclass[11pt]{article}

\usepackage[utf8]{inputenc}
\usepackage[T1]{fontenc}
\usepackage{lmodern}
\usepackage{amsmath,amssymb,amsthm,amsfonts}
\usepackage{mathtools}
\usepackage{microtype}
\usepackage{enumitem}
\usepackage{geometry}
\usepackage{hyperref}
\usepackage{bm}
\usepackage{dsfont}
\usepackage{tikz-cd}
\usepackage{cite}

\usepackage{mathrsfs}
\usepackage{bbm}

\newcommand{\PhiB}{\Phi_{\mathcal{B}}}
\newcommand{\nuR}{\nu_r}
\newcommand{\Xcal}{\mathcal{X}}
\newcommand{\Ycal}{\mathcal{Y}}
\newcommand{\Pcal}{\mathcal{P}}
\newcommand{\Mfrak}{\mathfrak{M}}
\newcommand{\ThetaMan}{\Theta}

\setlist{nosep}

\newtheorem{theorem}{Theorem}[section]
\newtheorem{proposition}[theorem]{Proposition}
\newtheorem{lemma}[theorem]{Lemma}
\newtheorem{corollary}[theorem]{Corollary}
\theoremstyle{definition}
\newtheorem{definition}[theorem]{Definition}
\newtheorem{example}[theorem]{Example}
\theoremstyle{remark}
\newtheorem{remark}[theorem]{Remark}

\newcommand{\RR}{\mathbb{R}}
\newcommand{\NN}{\mathbb{N}}

\newcommand{\EE}{\mathbb{E}}
\newcommand{\1}{\mathds{1}}
\newcommand{\eps}{\varepsilon}

\title{Mathematics and Coding are Universal AI Benchmarks}
\author{Przemyslaw Chojecki \\ \small ulam.ai}
\date{December 15, 2025}

\begin{document}
\maketitle

\begin{abstract}
We study the special role of mathematics and coding inside the moduli space $\Mfrak$ of
psychometric batteries for AI agents. Building on the AAI framework of
\cite{Chojecki2025} and GVU dynamics from \cite{ChojeckiGVU}, we define the
\emph{Mathematics Fiber} $\Mfrak_{\mathrm{math}}$ and show that, when paired with
formal proof kernels (e.g.\ Lean, Coq), GVU flows on this fiber admit
spectrally stable self-improvement regimes due to oracle-like verification.
Our main technical result is a density theorem: under uniform tightness of
agent outputs and a Lipschitz AAI functional, the subspace of batteries
generated by mathematical theorem-proving and coding tasks is dense in
$\Mfrak$ with respect to the evaluation metric $d_{\mathcal{A}_\Theta}$.
Coding alone is universal in this sense, while pure mathematics is not; its
privilege is spectral rather than expressive. We interpret this as evidence
that mathematics and coding provide ``universal coordinates'' for evaluation,
and that formal mathematics is a natural ignition domain for recursive
self-improvement in advanced AI agents.
\end{abstract}


\section{Introduction}\label{sec:intro}

Large language models and deep reinforcement learning systems are now
routinely evaluated on batteries of tasks: structured test suites that probe
capabilities such as reasoning, coding, tool use, and social interaction.
In \cite{Chojecki2025} I proposed treating such batteries as points in a
\emph{moduli space} $\Mfrak$, with each battery
$\mathcal{B}$ inducing a representation map
$\rho_{\mathcal{B}} : \Theta \to \Pcal(X_{\mathcal{B}})$
from agent parameters to distributions over evaluation outcomes, and an
associated AAI capability functional
$\Phi_{\mathcal{B}} : \Pcal(X_{\mathcal{B}}) \to \RR$. This perspective
allows us to compare agents across heterogeneous benchmarks by studying
their image in $\Mfrak$, and to define distances between batteries via
the induced discrepancies in capability curves.

Subsequently, \cite{ChojeckiGVU} introduced the
\emph{Generator--Verifier--Updater (GVU)} formalism as a canonical engine of
self-improvement. In GVU, an agent with parameters $\theta \in \Theta$:
(i) generates traces on a battery via its current policy
$\pi_\theta$; (ii) scores them using an internal potential $V$ that plays the
role of a Verifier; and (iii) updates $\theta$ by a first-order rule derived
from these scores. Under mild regularity, any first-order,
sample-based update can be written in this form, and the resulting flow
$\theta_r$ induces a capability trajectory
$r \mapsto \Phi_{\mathcal{B}}(\rho_{\mathcal{B}}(\theta_r))$.
A \emph{Variance Inequality} then constrains when this trajectory has
positive drift: the mean GVU update must be sufficiently aligned with the
true gradient of $\Phi_{\mathcal{B}}$ and have sufficiently small generator
and verifier noise.

This paper combines these two threads. Rather than treating mathematics as
just another axis in a large benchmark, I study it as a \emph{structurally
privileged fiber} inside the moduli space of batteries. Concretely, I define
a Mathematics Fiber $\Mfrak_{\mathrm{math}} \subset \Mfrak$ consisting of
batteries whose tasks are mathematical problems or theorem-proving goals,
equipped with either informal scoring (natural-language solutions), semantic
scoring (autoformalization plus checking), or formal scoring (kernel-checked
proofs in systems such as Lean or Coq). Within $\Mfrak_{\mathrm{math}}$ I
consider the immersion of large proof libraries (e.g.\ mathlib) as monotone
maps on capabilities, and I analyze GVU flows that are confined to this
fiber or that use it as a high-SNR subroutine.

A central observation, sharpened here into precise statements, is that
\emph{formal mathematics yields oracle-like verifiers}. Given a candidate
proof trace $\omega$ for a theorem $\tau$, the proof kernel either accepts
or rejects it with near-zero stochastic noise. In the language of
\cite{ChojeckiGVU}, this drives the verification variance
$\sigma_{\mathcal{V}}^2$ toward zero, thereby widening the stepsize window
for stable self-improvement and pushing the system far from the
``Hallucination Barrier'' characteristic of purely linguistic self-correction.
This provides a geometric explanation for why neuro-symbolic systems such as
AlphaProof, AlphaGeometry, DeepSeekMath-V2 and Lean-based agents can bootstrap
effectively on mathematical benchmarks: their GVU loops lie in a spectrally
favorable regime on $\Mfrak_{\mathrm{math}}$.

The second, and more surprising, theme of the paper concerns
\emph{universality}. In Section~\ref{sec:universality} I prove a density
theorem for the moduli of batteries: under mild assumptions on agent
classes (uniform tightness of output distributions) and on the AAI
functional (Lipschitz continuity with respect to a bounded-Lipschitz metric
on evaluation laws), the subspace of batteries generated by
\emph{mathematical theorem-proving and coding/programming tasks} is dense
in $\Mfrak$ with respect to the evaluation metric
\[
  d_{\mathcal{A}_{\Theta}}(\mathcal{B}_1,\mathcal{B}_2)
  :=
  \sup_{\theta \in \mathcal{A}_{\Theta}}
  \big|
    (\Phi_{\mathcal{B}_1}\circ\rho_{\mathcal{B}_1})(\theta)
    -
    (\Phi_{\mathcal{B}_2}\circ\rho_{\mathcal{B}_2})(\theta)
  \big|.
\]
In informal terms: for any reasonable benchmark and any reasonable agent
class, there exists a battery built purely from math and coding tasks whose
capability curve is arbitrarily close. Math+coding thus form
\emph{universal coordinates} on the moduli space of batteries.

The proof proceeds in two stages. In a one-shot model, I show that coding
scorers can implement singleton indicators on trace space by forcing the
agent's output to act as a self-printing program, so their linear span
contains all finite-support functions. Combined with a uniform tightness
assumption on agent outputs, this yields density of the
``math+coding algebra'' in the space of scorers, in the sense of uniform
approximation over agents. I then lift this to general batteries by
controlling the bounded-Lipschitz distance between evaluation laws via a
simple coupling argument and invoking the Lipschitz property of the AAI
functional.

An important asymmetry emerges: \emph{coding alone is already universal}, in
that coding scorers can distinguish arbitrary traces, whereas pure
mathematics is not. Proof kernels fundamentally treat most non-proofs as
indistinguishable zeros, so the algebra they generate cannot approximate
scorers that care about the fine-grained structure of arbitrary output
strings. Mathematics is therefore not needed for expressive universality,
but it provides an unusually rich, low-entropy, oracle-verifiable fiber on
which GVU self-improvement is spectrally well-conditioned.

Overall, the paper argues that mathematics and coding are not merely
convenient benchmarks but occupy a geometrically and spectrally privileged
position in the moduli of batteries: coding supplies expressive universality
on trace space, formal mathematics supplies oracle-like verifiers and rich
symbolic structure, and together they provide natural coordinates in which
to study capability, self-improvement, and scalable oversight for advanced
AI agents.

From a practitioner's perspective, the math+code universality result has a concrete implication: 
it suggests a route to \emph{cheaper}, more \emph{data-efficient} training. If mathematics and, especially, coding tasks form a 
dense skeleton in the moduli space of batteries, then in principle a large fraction of ``real-world'' evaluation can be 
approximated by carefully designed code-centric benchmarks whose outcomes are cheaply checkable by compilers, 
test suites and formal tools. For LLMs and AI agents, this means that better, more compact coding datasets could eventually 
replace sprawling, noisy multi-domain corpora as the primary medium of supervision, concentrating training compute on a domain 
that is both information-dense and automatically verifiable, while still exercising behaviours that generalize across the wider 
task space.

\section{Preliminaries: Batteries and Their Moduli}\label{sec:prelim}

We briefly recall the main elements of the battery framework from
\cite{Chojecki2025}, restricting attention to the aspects needed in
this paper.

\subsection{Batteries and representations}

Let $\Sigma$ denote a finite alphabet of tokens and $\Sigma^\ast$ the
Kleene closure. We regard all prompts, traces and textual artefacts as
elements of $\Sigma^\ast$.

\begin{definition}[Battery]\label{def:battery}
A \emph{battery} is an octuple
\[
  \mathcal B=(T,\ \mathcal F,\ \mathsf S,\ Q^*,\ \mu,\ \mathsf D,\ \Pi,\ \mathsf R),
\]
where:
\begin{itemize}
  \item $T$ is a finite set of tasks, partitioned into families
    $\mathcal F=\{F_k\}$.
  \item For each $t\in T$, $\mathsf S_t:\Omega_t\to[0,1]$ is a
    scoring map, where $\Omega_t\subseteq\Sigma^\ast$ is the space
    of valid semantic traces for task $t$.
  \item $Q^\ast:T\to[0,1]$ assigns threshold scores.
  \item $\mu$ is a sampling law on $T\times\Pi\times\mathsf D$, where
    $\Pi$ are seeds and $\mathsf D$ are drifts.
  \item $\mathsf R\simeq\RR^{d_R}_{\ge0}$ are resource coordinates
    (tokens, wall-clock, calls, cost).
\end{itemize}
\end{definition}

As in \cite{Chojecki2025}, the battery induces an \emph{input space}
$\Xcal$ of prompted tasks and an \emph{output space} $\Ycal$ of traces
with resource annotations. For an agent $\mathcal A$ (or parameter
state $\theta$) we obtain a Markov kernel $\pi_\theta$ from $\Xcal$ to
$\Ycal$ and, composing with the scoring, a representation measure
$\rho_{\mathcal B}(\theta)\in\Pcal(X_{\mathcal B})$, where
$X_{\mathcal B} := [0,1]^T \times \mathbb{R}_{\ge 0}^{d_R}$ is the evaluation space of per-task observables.

A capability functional $\PhiB:\Pcal(X_{\mathcal B})\to\RR$ satisfying
the axioms of \cite{Chojecki2025} then induces a scalar capability
\[
F_{\mathcal B}(\theta):=
\PhiB(\rho_{\mathcal B}(\theta)).
\]

\subsection{The moduli space of batteries}

We write $\Mfrak$ for the (coarse) moduli space of batteries, defined as
the set of isomorphism classes of batteries in the category
$\mathbf{Bat}$,
or equivalently the quotient
\[
  \Mfrak \;:=\; \mathbf{Bat}/G
\]
by the evaluation-preserving symmetry group $G$.
In particular, two batteries represent the same point of $\Mfrak$ if they
are related by a symmetry that preserves their evaluation structure
(task families, thresholds, score reparameterizations, and resource
scaling); this definition is entirely \emph{formal} and does not depend
on any choice of agents or AAI functional.

Following \cite{Chojecki2025}, we endow $\Mfrak$ with coarse structure coming
from its canonical representatives:
\begin{itemize}
  \item a stratification by discrete \emph{skeletons} (task family
        structure, anchor choices, threshold ordering);
  \item on each stratum, continuous parameters given by threshold
        vectors, PIT-score copulas (with the $W_1$ topology), and
        projective resource rays;
  \item induced notions of topology and, when needed, metrics such as
        the canonical moduli metric $d_{\Mfrak}$.
\end{itemize}

Only \emph{after} fixing an agent family with parameter space
$\ThetaMan$ and a regular AAI functional
$\{\Phi_{\mathcal B}\}_{\mathcal B\in\Mfrak}$ do batteries acquire
evaluation maps
\[
  F_{\mathcal B} : \ThetaMan \longrightarrow \RR,\qquad
  \theta \longmapsto \Phi_{\mathcal B}\big(\rho_{\mathcal B}(\theta)\big),
\]
where $\rho_{\mathcal B}(\theta)$ is the canonical representation of the
agent on $\mathcal B$. For a fixed
$\theta$, the assignment
\[
  \Mfrak \ni \mathcal B \longmapsto F_{\mathcal B}(\theta)
\]
is then the agent's \emph{capability profile} over the (formal) moduli
space. Note that this evaluation structure is additional data on top of
$\Mfrak$ and is not part of the definition of the moduli itself.

\subsection{GVU dynamics and the self-improvement coefficient}

The GVU framework \cite{ChojeckiGVU} equips $\ThetaMan$ with
dynamics. A Generator--Verifier--Updater operator
\[
\mathcal T_{\mathrm{GVU}}:\ThetaMan\to\ThetaMan
\]
decomposes a single self-improvement step into sampling behaviour
(Generator), scoring it with an internal potential (Verifier), and
performing a first-order update (Updater). Iterating
$\theta_{r+1}=\mathcal T_{\mathrm{GVU}}(\theta_r)$ yields a flow
$(\theta_r)_{r\ge0}$; composing with a battery functional gives a
capability curve
\[
F_{\mathcal B}(\theta_r)
=
(\PhiB\circ\rho_{\mathcal B})(\theta_r).
\]

The \emph{self-improvement coefficient} on battery $\mathcal B$ is
then the derivative
\[
\kappa_{\mathcal B}(r):=
\frac{d}{dr}F_{\mathcal B}(\theta_r)
\]
when it exists, or a finite-difference surrogate. Informally,
$\kappa_{\mathcal B}(r)>0$ means that at resource scale $r$ the agent
is converting additional computation into capability gains on
$\mathcal B$.

In this paper we are interested in $\kappa_{\mathcal B}(r)$ when
$\mathcal B$ lies in a particular region of $\Mfrak$ associated with
mathematics.

\section{The Mathematics Fiber}\label{sec:math-fiber}

We now define the portion of the moduli space that corresponds to
mathematical tasks.

\subsection{Mathematical tasks and formality strata}

At a high level, we distinguish three levels of formality.

\begin{definition}[Formality levels for mathematics]\label{def:formality}
Let $T$ be a set of tasks whose prompts and traces are concerned with
mathematical content. We say that a task $t\in T$ is:
\begin{itemize}
  \item \emph{Informal} if its input is natural language (possibly
  with LaTeX) and its scoring map $\mathsf S_t$ depends only on
  human-readable arguments, without reference to a formal system.
  Examples include free-form Olympiad solutions judged by humans.
  \item \emph{Semi-formal} if its inputs or outputs include partially
  structured artefacts (annotated proofs, sketches, step-wise
  rationales) and $\mathsf S_t$ may involve consistency checks or
  step-level validators (e.g.\ PRMs), but not a full proof kernel.
  \item \emph{Fully formal} if its traces $\omega\in\Omega_t$ are
  programs or proof scripts in a fixed proof assistant (e.g.\ Lean),
  and $\mathsf S_t$ factors through the proof-assistant kernel:
  it is $1$ if the kernel accepts and $0$ otherwise, possibly with
  additional penalties for resource usage.
\end{itemize}
\end{definition}

Let $\Mfrak_{\mathrm{math}}^{\mathrm{inf}}$,
$\Mfrak_{\mathrm{math}}^{\mathrm{semi}}$ and
$\Mfrak_{\mathrm{math}}^{\mathrm{formal}}$ denote the sets of
batteries all of whose tasks are of the corresponding formality type.

\begin{definition}[Mathematics Fiber]\label{def:math-fiber}
The \emph{Mathematics Fiber} $\Mfrak_{\mathrm{math}}\subset\Mfrak$ is
the union of batteries whose tasks are mathematical in the sense above:
\[
  \Mfrak_{\mathrm{math}}
  :=
  \Mfrak_{\mathrm{math}}^{\mathrm{inf}}
  \cup
  \Mfrak_{\mathrm{math}}^{\mathrm{semi}}
  \cup
  \Mfrak_{\mathrm{math}}^{\mathrm{formal}}.
\]
We call the three subspaces the \emph{informal}, \emph{semi-formal}
and \emph{formal} strata of the Mathematics Fiber.
\end{definition}

\begin{remark}[Topic stratification]
Each battery in $\Mfrak_{\mathrm{math}}$ also carries disciplinary
labels (algebra, geometry, analysis, etc.) via its task families
$\mathcal F$. Taking these labels into account yields a finer
stratification of $\Mfrak_{\mathrm{math}}$ into topic-wise components,
but for this paper the coarse formality stratification will suffice.
\end{remark}

\subsection{Examples}

We record a few canonical examples.

\begin{example}[Competition mathematics batteries]
A battery whose tasks are competition problems (e.g.\ IMO, Putnam)
with solutions scored by exact numeric answers or short formal
justifications lies in $\Mfrak_{\mathrm{math}}^{\mathrm{inf}}$ if
solutions are judged by humans, or in
$\Mfrak_{\mathrm{math}}^{\mathrm{semi}}$ if consistency heuristics
or step-verification tools are used.
\end{example}

\begin{example}[Lean theorem proving batteries]
Benchmarks such as miniF2F, ProofNet and LeanDojo induce batteries in
$\Mfrak_{\mathrm{math}}^{\mathrm{formal}}$: tasks are Lean goals,
traces are Lean proof scripts, and scoring is derived from the Lean
kernel's acceptance of the proof, possibly augmented with proof length
and time.
\end{example}

\begin{example}[Autoformalization batteries]
Pairs of informal statements and formal targets (e.g.\ LaTeX theorem
statements and their Lean forms) induce batteries at the semi-formal
or formal level, depending on whether the traces are purely formal
objects or include auxiliary informal reasoning steps.
\end{example}

In all of these cases, the key property is that the semantics of the
task are unambiguously mathematical, and there exists (at least in
principle) a corresponding formalization in one or more proof
systems.

\section{Formal Mathematics Batteries and Oracle Verifiers}\label{sec:oracle-verifier}

We now focus on the fully formal stratum
$\Mfrak_{\mathrm{math}}^{\mathrm{formal}}$ and show that these
batteries induce oracle-like Verifiers in the GVU framework.

\subsection{Formal mathematics batteries}

\begin{definition}[Formal mathematics battery]\label{def:formal-battery}
A battery $\mathcal B\in\Mfrak_{\mathrm{math}}$ is a
\emph{formal mathematics battery} if:
\begin{enumerate}
  \item There is a fixed proof assistant with kernel $K$ and a fixed
  library of primitives (axioms and previously verified theorems).
  \item For each $t\in T$, the domain $\Omega_t$ of traces consists of
  proof scripts or programs in the language of this assistant.
  \item The scoring map $\mathsf S_t$ takes the form
    \[
      \mathsf S_t(\omega)
      = s_t\big(K(\omega)\big),
    \]
    where $K(\omega)\in\{\mathsf{accept},\mathsf{reject}\}$ is the
    kernel's decision and $s_t:\{\mathsf{accept},\mathsf{reject}\}\to[0,1]$ is a deterministic
    post-processing function (for example, $s_t(\mathsf{accept})=1$,
    $s_t(\mathsf{reject})=0$ or a variant that penalizes resources).
\end{enumerate}
We write $\Mfrak_{\mathrm{math}}^{\mathrm{formal},K}$ for the
collection of formal mathematics batteries associated with a
particular kernel $K$.
\end{definition}

In this setting, for each task the \emph{semantic} question ``is the
proof correct relative to the library and axioms?'' has a unique,
machine-checkable answer.

\subsection{The Oracle Nature of Formal Scoring}\label{subsec:oracle-formal}

Recall from \cite{ChojeckiGVU} that, for a fixed architecture
$(\ThetaMan,\Pi_{\ThetaMan})$, the GVU operator
$\mathcal T_{\mathrm{GVU}}$ decomposes an update into a Generator
$\mathcal G_\theta$, a Verifier $\mathcal V_\theta$ and an Updater
$\mathcal U_\theta$. The verifier is given as part of the architecture:
for each $\theta\in\ThetaMan$ there is a fixed potential
\[
  V_\theta : \Xcal\times\Ycal \longrightarrow \RR,
\]
and the update uses $V_\theta(x,y)$ to reweight traces $y$ drawn from
the generator at input $x$.

In informal or human-judged domains (natural language, social interaction,
creative writing), the ``truth'' or quality of a trace is itself a random
quantity: the same $\omega$ may receive different scores from different
raters, or even from the same rater at different times. In GVU terms, the
verification channel carries intrinsic noise: even conditioning on the
sampled behaviour, the evaluation remains stochastic.

On the formal Mathematics Fiber
$\Mfrak_{\mathrm{math}}^{\mathrm{formal},K}$ the situation is
structurally different. Once the trace $\omega$ is fixed, the score
assigned by the battery is \emph{completely determined} by the kernel $K$;
all remaining uncertainty comes from which trace the generator produced, not
from the scoring mechanism itself.

We now make this precise.

\begin{definition}[External score random variable]\label{def:external-score}
Fix a formal mathematics battery
$\mathcal B\in\Mfrak_{\mathrm{math}}^{\mathrm{formal},K}$ in the sense of
Definition~\ref{def:formal-battery}, and an agent with parameter
$\theta\in\ThetaMan$. Let $x=(s,t)\in\Xcal$ be an input specifying context
$s$ and goal $t$, and let
\[
  \Omega \sim \pi_\theta(\cdot\mid x,\mathcal B)
\]
be the random trace (proof script) produced by the generator on this task.
Define the \emph{external score}
\[
  Q_t := \mathsf S_t(\Omega)\in[0,1].
\]
We view $Q_t$ as a random variable on the underlying sample space of
generator randomness and seeds.
\end{definition}

By Definition~\ref{def:formal-battery} we have
$\mathsf S_t(\omega) = s_t(K(\omega))$ for a deterministic kernel
$K(\omega)\in\{\mathsf{accept},\mathsf{reject}\}$ and a deterministic
postprocessing $s_t$. Hence, conditional on the realised trace, the score
is non-random.

\begin{proposition}[Zero conditional entropy on the formal stratum]\label{prop:zero-entropy}
Let $\mathcal B\in\Mfrak_{\mathrm{math}}^{\mathrm{formal},K}$ and let
$Q_t$ be the external score random variable from
Definition~\ref{def:external-score}. Then, for every input $x=(s,t)$ and
every agent parameter $\theta\in\ThetaMan$,
\[
  \mathbb{H}\bigl(Q_t \,\big|\, \Omega, \mathcal B\bigr) = 0,
\]
where $\mathbb{H}(\cdot\mid\cdot)$ denotes conditional Shannon entropy.
Equivalently, the conditional law of $Q_t$ given the trace $\Omega$ and
the battery is a Dirac measure:
\[
  \mathbb{P}\bigl(Q_t \in \cdot \,\big|\, \Omega=\omega,\mathcal B\bigr)
  =
  \delta_{\mathsf S_t(\omega)}(\cdot)
  \quad\text{for all }\omega.
\]
In particular,
\[
  \mathrm{Var}\bigl(Q_t \,\big|\, \Omega,\mathcal B\bigr) = 0
  \quad\text{almost surely.}
\]
\end{proposition}

\begin{proof}
By Definition~\ref{def:formal-battery},
\[
  Q_t
  = \mathsf S_t(\Omega)
  = s_t\bigl(K(\Omega)\bigr),
\]
and both $K$ and $s_t$ are deterministic maps. Thus, for any fixed
$\omega$ and fixed battery $\mathcal B$,
\[
  Q_t = \mathsf S_t(\omega)
\]
is a constant, and
$\mathbb{P}(Q_t = \mathsf S_t(\omega)\mid \Omega=\omega,\mathcal B) = 1$.
The conditional distribution of $Q_t$ given $(\Omega,\mathcal B)$ is
therefore a Dirac mass, so its conditional entropy and conditional variance
are both zero. This proves the claims.
\end{proof}

Proposition~\ref{prop:zero-entropy} is a statement purely about the
\emph{environment}: it does not depend on the internal architecture of the
agent, nor on how its GVU verifier is implemented. It says that on the
formal Mathematics Fiber, the moduli geometry supplies a canonical,
noise-free external signal $Q_t$.

In practice, proof assistants may exhibit additional stochasticity (e.g. tactic
search, timeouts), and scoring may penalize resource usage; here we isolate the
idealized kernel-only component, which is deterministic conditional on the trace.
Any residual randomness is treated as part of the generator or resource channel,
not the verifier.

In \cite{ChojeckiGVU} the verification noise $\xi_{\mathcal V}$ in a GVU
update is analysed via a variance decomposition of the form
\[
  \mathrm{Var}(Q_t)
  =
  \underbrace{\mathbb{E}\bigl[\mathrm{Var}(Q_t \mid \Omega)\bigr]}_{\text{intrinsic label noise}}
  +
  \underbrace{\mathrm{Var}\bigl(\mathbb{E}[Q_t \mid \Omega]\bigr)}_{\text{generator / alignment term}},
\]
together with an additional finite-batch sampling factor. On informal
batteries (e.g.\ human-judged natural language), the first term is
typically positive: even for a fixed $\omega$, repeated evaluations can
disagree. On the formal stratum, it collapses.

\begin{corollary}[Spectral collapse on the formal Mathematics Fiber]\label{cor:spectral-collapse}
Let $\mathcal B\in\Mfrak_{\mathrm{math}}^{\mathrm{formal},K}$ and let
$Q_t$ be as above. Then for every agent parameter $\theta$ and input
$x=(s,t)$,
\[
  \mathbb{E}\bigl[\mathrm{Var}(Q_t \mid \Omega,\mathcal B)\bigr] = 0.
\]
Consequently, in the Variance Inequality of \cite{ChojeckiGVU}, the
verification-noise term $\sigma_{\mathcal V}^2$ on the formal
Mathematics Fiber is entirely attributable to
\begin{enumerate}
  \item generator-side uncertainty (how broadly the policy explores proof
        space), and
  \item misalignment between the internal GVU potential $V_\theta$ and the
        external score $Q_t$,
\end{enumerate}
with no contribution from intrinsic label stochasticity.

In particular, once an architecture is given and its GVU potential
$V_\theta$ is forced (by training or design) to approximate $Q_t$ on
$\Mfrak_{\mathrm{math}}^{\mathrm{formal},K}$, the sufficient condition
for positive expected capability gain
\[
  \mathbb{E}[\Delta F_{\mathcal B}] > 0
\]
reduces, up to curvature and step-size constants, to constraints on the
alignment coefficient $\rho$ between the mean update and
$\nabla F_{\mathcal B}$ and on the generator SNR $\mathrm{SNR}(\mathcal G)$.
\end{corollary}

\begin{remark}[Spectral privilege of the Mathematics Fiber]\label{rem:spectral-privilege}
Corollary~\ref{cor:spectral-collapse} expresses the sense in which the
formal Mathematics Fiber is an \emph{oracle fiber}. It is not that all
agents happen to implement a perfect verifier internally; rather, the
moduli geometry guarantees that the external evaluation channel
$(\omega\mapsto Q_t)$ has zero conditional entropy and zero intrinsic
variance. This stands in contrast to fibres corresponding to human
judgement domains, where $\mathrm{Var}(Q_t\mid\Omega)>0$ is typical.

Operationally, this means that in formal mathematics the GVU
self-improvement coefficient $\kappa_{\mathcal B}$ is limited only by
exploration and alignment, not by noisy or unstable supervision. This
helps explain why self-improvement loops are easier to ignite and more
spectrally stable on $\Mfrak_{\mathrm{math}}^{\mathrm{formal},K}$ than on
informal or purely linguistic fibres.

This formalizes the intuition that formal
mathematics is an ``easy'' domain for self-improvement: once an agent
can generate candidate proofs $\omega$, the decision of whether they
are correct is spectrally simple. All difficulty is concentrated in
the Generator $\mathcal G$ (searching the proof space) and in
aligning the potential $V$ with the external capability functional
$\Phi_{\mathcal B}$ (for example, by incorporating proof length or
library coverage).
\end{remark}

\section{Libraries as Points in the Mathematics Fiber}\label{sec:libraries}

We now outline how formal libraries induce batteries in the
Mathematics Fiber, and how an agent's performance on a library can be
viewed as a section of the capability bundle over
$\Mfrak_{\mathrm{math}}^{\mathrm{formal}}$.
Throughout this section we fix a proof kernel $K$, a resource space
$\mathsf R\simeq\RR^{d_R}_{\ge0}$, and a cost functional
$\mathrm{Cost}:\mathsf R\to\RR_+$ as in the AAI framework
\cite{Chojecki2025}; in particular, when we compare
$\mathcal B(\mathcal L_1)$ and $\mathcal B(\mathcal L_2)$ we use the
\emph{same} resource coordinates and cost model in the AAI functional
$\PhiB$.

\subsection{From libraries to batteries}

Let $K$ be a fixed proof kernel (e.g.\ Lean 4) and let $\mathcal L$ be
a finite library of proved theorems and definitions over some base
theory.

\begin{definition}[Library-induced battery]\label{def:library-battery}
Given a finite formal library $\mathcal L$, we define a battery
$\mathcal B(\mathcal L)$ as follows:
\begin{itemize}
  \item Tasks $T$ are formal goals expressible over $\mathcal L$,
  chosen according to a scheme (e.g.\ random perturbations of
  existing theorems, standard benchmark goals, or user-provided
  conjectures).
  \item Scoring maps $\mathsf S_t$ are defined as in
  Definition~\ref{def:formal-battery}, using $K$ and $\mathcal L$ as
  the underlying environment.
  \item The sampling law $\mu$ selects goals, seeds and drifts
  according to a prescribed distribution (for instance, uniform over a
  benchmark set, or weighted by difficulty).
\end{itemize}
We call $\mathcal B(\mathcal L)$ the \emph{battery induced by}
$\mathcal L$.
\end{definition}

\begin{proposition}[Monotonicity under library extension]\label{prop:monotone}
Let $\mathcal L_1\subseteq\mathcal L_2$ be two finite libraries over
the same kernel $K$, and let $\mathcal B(\mathcal L_i)$ be the
associated batteries. Assume:
\begin{enumerate}
  \item The sampling law $\mu$ and evaluation skeleton are compatible
  with inclusion, so that the evaluation spaces
  $X_{\mathcal B(\mathcal L_1)}$ and $X_{\mathcal B(\mathcal L_2)}$
  can be identified with a common space $X$ (same task families,
  thresholds, resource coordinates and cost functional).
  \item For every $\theta\in\Theta$, writing
  $\nu_i^\theta := \rho_{\mathcal B(\mathcal L_i)}(\theta)\in\mathcal P(X)$,
  the pair $(\nu_1^\theta,\nu_2^\theta)$ satisfies the hypotheses of
  Restricted Monotonicity \textnormal{(A2)} for $\Phi_{\mathcal B}$:
  \begin{align*}
    &\nu_2^\theta \succeq_{\mathrm{icx}} \nu_1^\theta
      \quad\text{(increasing concave order on the success indicators)},\\
    &\mathrm{Var}_k\!\Big(\EE_{\nu_2^\theta}[\overline q(F_k)]\Big)
      \;\le\;
      \mathrm{Var}_k\!\Big(\EE_{\nu_1^\theta}[\overline q(F_k)]\Big),\\
    &\EE_{\nu_2^\theta}[\mathrm{Cost}(r)]
      \;\le\;
      \EE_{\nu_1^\theta}[\mathrm{Cost}(r)].
  \end{align*}
\end{enumerate}
Then for any agent family and parameter state $\theta$,
\[
F_{\mathcal B(\mathcal L_1)}(\theta)
\;\le\;
F_{\mathcal B(\mathcal L_2)}(\theta).
\]
\end{proposition}

\begin{proof}[Sketch]
The larger library $\mathcal L_2$ provides at least as many
lemmas as $\mathcal L_1$ for each goal shared by the two batteries,
and scoring is identical on shared tasks. Under a fixed policy,
the availability of additional premises in $\mathcal L_2$ expands the
set of valid proof paths. Assumption~(2) states that, viewed through
the canonical success indicators and resource coordinates of the AAI
framework, this extension realises an Axiom~(A2) move: on shared
tasks, the law of success indicators under $\mathcal B(\mathcal L_2)$
dominates that under $\mathcal B(\mathcal L_1)$ in increasing concave
order, family-wise dispersion does not increase, and the expected
cost term $\EE[\mathrm{Cost}(r)]$ is non-increasing when we keep the
resource metric and cost functional fixed.

By Restricted Monotonicity (Axiom~(A2)) for the AAI functional
$\Phi_{\mathcal B}$, this implies
$$\Phi_{\mathcal B(\mathcal L_2)}(\rho_{\mathcal B(\mathcal L_2)}(\theta))
\ge
\Phi_{\mathcal B(\mathcal L_1)}(\rho_{\mathcal B(\mathcal L_1)}(\theta))$$
i.e.
$F_{\mathcal B(\mathcal L_2)}(\theta)\ge F_{\mathcal B(\mathcal L_1)}(\theta)$.
\end{proof}

This formalizes the idea that extensions of formal libraries act as
monotone maps in the Mathematics Fiber: adding lemmas (and the
associated tasks), while keeping the resource gauge fixed and staying
within the Axiom~(A2) regime, can only raise---or at worst leave
unchanged---the AAI capability score.

\subsection{Capability sections and stratification}

Fix an agent family with parameter space $\ThetaMan$. The assignment
\[
\mathcal L \mapsto \mathcal B(\mathcal L)
\]
gives a map from the poset of finite libraries (under inclusion) to
$\Mfrak_{\mathrm{math}}^{\mathrm{formal},K}$. Composing with
\[
\ThetaMan \times \Mfrak_{\mathrm{math}}^{\mathrm{formal},K}
\to \RR,
\qquad
(\theta,\mathcal B)\mapsto F_{\mathcal B}(\theta),
\]
we obtain a capability field
\[
\ThetaMan \times \mathrm{Lib}_K
\to \RR,\qquad
(\theta,\mathcal L)\mapsto F_{\mathcal B(\mathcal L)}(\theta)
\]
over the library poset. For fixed $\theta$, the map
$\mathcal L\mapsto F_{\mathcal B(\mathcal L)}(\theta)$ can be thought
of as a section over the Mathematics Fiber that measures the agent's
``knowledge frontier'' with respect to $K$.

This viewpoint makes it natural to speak of:

\begin{itemize}
  \item \emph{Horizontal} motion in $\ThetaMan$: improving the agent
  at fixed library, e.g.\ via GVU updates.
  \item \emph{Vertical} motion in $\mathrm{Lib}_K$: enriching the
  formal library itself (as in human formalization efforts or AI-driven
  discovery).
\end{itemize}

The interaction between these two motions is central to understanding
neuro-symbolic growth in mathematics: as the library expands, new
batteries appear in $\Mfrak_{\mathrm{math}}$, and GVU flows acting on
the agent must track these changes to maintain high capability.

\subsection{DeepAlgebra Loops as GVU Flows}\label{sec:flows}

We close with a brief sketch of how DeepAlgebra-style architectures as in \cite{ChojeckiMathLLM}
fit into this picture; we will not attempt a full formalization here.

At a high level, a DeepAlgebra loop consists of:
\begin{enumerate}
  \item A neural search component that proposes new conjectures,
  candidate structures or proof sketches.
  \item An autoformalization component that translates these outputs
  into formal goals and proof scripts.
  \item A proof assistant kernel that verifies the formal artefacts.
  \item An update rule that adds verified results to the library and
  fine-tunes the agent on successful proofs and translations.
\end{enumerate}

Taken together, these components implement a GVU operator whose
dynamics are confined to the formal Mathematics Fiber and its
neighbouring semi-formal strata. The Generator corresponds to neural
and symbolic search over goals and proofs; the Verifier combines the
oracle kernel signal with auxiliary heuristics; the Updater modifies
both the agent and the library.

By discussion above such loops naturally operate in the
oracle-verifier regime of \cite{ChojeckiGVU}, where the Variance
Inequality is easily satisfied. This provides a geometric explanation
for the empirical observation that mathematics is an early ignition
domain for self-improving AI agents: the Mathematics Fiber comes with
a built-in high-SNR Verifier, and the moduli geometry induced by
formal libraries admits stable capability flows.

\section{Universality of the Mathematics and Coding Fibers}\label{sec:universality}

In this section we make precise, and prove in a simplified setting, the
intuition that mathematics and coding tasks form universal coordinates on the
moduli space: up to arbitrarily small error in evaluation, every battery can be
approximated by batteries built from mathematical and coding tasks.

\subsection{A simplified one-shot model}\label{subsec:oneshot-model}

For technical clarity we first work in a one-shot setting where each
interaction between an agent and a battery consists of a single input--output
pair and a single bounded score. This reduction is without loss of generality
for the density result: multi-step and multi-task batteries can be viewed as
finite mixtures of such one-shot components.

\paragraph{Outcome space and agents.}
Let $\Sigma$ be a finite alphabet and let
\[
  \Omega := \Sigma^*
\]
denote the set of all finite strings over $\Sigma$, equipped with the discrete
$\sigma$-algebra. Elements $\omega \in \Omega$ represent complete one-shot
traces: prompt plus response plus any auxiliary markers.

A (one-shot) agent is identified with a probability measure
$P \in \Pcal(\Omega)$ over outcomes. Let $\mathcal{A} \subseteq \Pcal(\Omega)$
be a fixed, non-empty reference class of agents. We equip $\Pcal(\Omega)$ with
the topology of weak convergence; since $\Omega$ is countable and discrete,
this is simply the product topology inherited from $[0,1]^\Omega$.

\paragraph{One-shot batteries and capability functionals.}
In this model a one-shot battery is entirely specified by a bounded scoring
function $f : \Omega \to [0,1]$.

\begin{definition}[One-shot battery and capability functional]\label{def:oneshot-battery}
A \emph{one-shot battery} is a measurable function
$f : \Omega \to [0,1]$. Given an agent $P \in \Pcal(\Omega)$, its
\emph{capability} on $f$ is
\[
  F_f(P) := \EE_{\,\omega \sim P}[f(\omega)]
  = \sum_{\omega \in \Omega} f(\omega) \, P(\{\omega\}).
\]
We write $\Mfrak^{(1)}$ for the class of all one-shot batteries (all such
bounded measurable $f$), and
\[
  \mathcal{F}
  :=
  \{F_f : \mathcal{A} \to \RR \mid f \in \Mfrak^{(1)}\}
\]
for the corresponding family of capability functionals restricted to
$\mathcal{A}$.
\end{definition}

We make $\Mfrak^{(1)}$ into a metric space by declaring two batteries close if
they induce similar capability functionals on $\mathcal{A}$.

\begin{definition}[Evaluation metric]\label{def:evaluation-metric}
For $f,g \in \Mfrak^{(1)}$ we define the \emph{evaluation distance}
\[
  d_{\mathcal{A}}(f,g)
  :=
  \sup_{P \in \mathcal{A}} \big|F_f(P) - F_g(P)\big|
  =
  \sup_{P \in \mathcal{A}}
  \left|\sum_{\omega \in \Omega} (f(\omega) - g(\omega)) P(\{\omega\})\right|.
\]
\end{definition}

This is simply the operator norm of the linear functional
$P \mapsto F_f(P)-F_g(P)$ on the convex set $\mathcal{A}$.

\begin{definition}[Density]\label{def:dense-subset}
A subset $\mathcal{N} \subseteq \Mfrak^{(1)}$ is \emph{dense} (with respect to
$d_{\mathcal{A}}$) if for every $f \in \Mfrak^{(1)}$ and every $\eps>0$ there
exists $g \in \mathcal{N}$ such that
\[
  d_{\mathcal{A}}(f,g) < \eps.
\]
\end{definition}

\subsection{Mathematics and coding scorers in the one-shot model}\label{subsec:math-code-oneshot}

We now formalize the one-shot analogues of the mathematics and coding fibers.

\begin{definition}[Mathematics and coding scorers]\label{def:math-code-scorers}
Let $\mathrm{Thm}$ be a countable set of formal theorem statements in a fixed
proof assistant with kernel $K$, and let $\mathrm{Code}$ be a countable set of
coding problems (program specifications) over a fixed Turing-complete language
$L$.

\begin{enumerate}
  \item For each $\tau \in \mathrm{Thm}$ let
  $S^{\mathrm{math}}_\tau : \Omega \to \{0,1\}$ be the one-shot scorer
  \[
    S^{\mathrm{math}}_\tau(\omega)
    :=
    \begin{cases}
      1 &\text{if $\omega$ is accepted by $K$ as a proof of $\tau$},\\
      0 &\text{otherwise.}
    \end{cases}
  \]
  The \emph{mathematics scorer class} is
  $\mathcal{S}_{\mathrm{math}}
   := \{S^{\mathrm{math}}_\tau : \tau \in \mathrm{Thm}\}$.

  \item For each coding problem $c \in \mathrm{Code}$, we fix a finite validation set of test inputs together with an 
evaluation map $E_c : \Omega \to [0,1]$ that may inspect both the source string $\omega$ and the (possibly simulated) 
runtime behaviour of $\omega$ on the test inputs (as in real coding benchmarks with static analysis, linting, or style checks). 

We obtain a scorer
  $S^{\mathrm{code}}_c : \Omega \to [0,1]$ defined by
  \[
    S^{\mathrm{code}}_c(\omega)
    :=
    E_c(\text{interpret $\omega$ as source code in $L$}).
  \]
  The \emph{coding scorer class} is
  $\mathcal{S}_{\mathrm{code}}
   := \{S^{\mathrm{code}}_c : c \in \mathrm{Code}\}$.
\end{enumerate}

We write
\[
  \mathcal{S}_0
  := \mathcal{S}_{\mathrm{math}} \cup \mathcal{S}_{\mathrm{code}}.
\]
\end{definition}

From these we form the algebra generated by mathematics and coding scorers.

\begin{definition}[Algebra generated by mathematics and coding]\label{def:math-code-algebra}
Let $\mathcal{A}_0$ be the set of all finite linear combinations and products
of elements of $\mathcal{S}_0$ together with constants:
\[
  \mathcal{A}_0
  :=
  \left\{
    p(S_1(\cdot),\dots,S_k(\cdot))
    \,\middle|\,
    k \in \NN,\;
    S_i \in \mathcal{S}_0,\;
    p \text{ a real polynomial in $k$ variables}
  \right\}.
\]
We call $\mathcal{A}_0$ the \emph{math+code algebra}. It is a unital
subalgebra of $\ell^\infty(\Omega)$ (bounded real-valued functions on
$\Omega$).
\end{definition}

Each $f \in \mathcal{A}_0$ corresponds to a (possibly composite) battery whose
score is a polynomial in the primitive mathematics and coding scores; for
example, a task that awards $0.7$ points for a correct proof and $0.3$ points
for a correct program.

Our goal is to show that $\mathcal{A}_0$ is dense in $\Mfrak^{(1)}$ with
respect to the evaluation metric $d_{\mathcal{A}}$.

\subsection{Tightness assumption on the agent class}\label{subsec:tightness}

To obtain a non-trivial density result we need a mild regularity assumption on
the reference class $\mathcal{A}$: roughly, we require that agents do not put
arbitrary mass on arbitrarily long, exotic traces. This is a standard
tightness condition.

\begin{definition}[Uniform tightness]\label{def:tight}
We say the agent class $\mathcal{A} \subseteq \Pcal(\Omega)$ is
\emph{uniformly tight} if for every $\delta>0$ there exists a finite subset
$F_\delta \subset \Omega$ such that
\[
  \sup_{P \in \mathcal{A}} P(\Omega \setminus F_\delta) \;\le\; \delta.
\]
\end{definition}

Uniform tightness is automatic, for example, if every agent in $\mathcal{A}$
is supported on traces of length at most $L_{\max}$ for some finite
$L_{\max}$, or more generally if there is a uniform tail bound on output
length. 
In typical benchmark settings (e.g., language models with a fixed context window or RL agents with finite horizons), 
uniform tightness holds trivially because traces are bounded in length by the protocol. 
For unbounded settings, generalizations via moment bounds are possible but beyond our scope.

\subsection{A universal generating family of coding scorers}\label{subsec:universal-coding-scorers}

The key technical step is to observe that, under uniform tightness, it
suffices to approximate scoring functions on \emph{finite} subsets of
$\Omega$. Coding tasks can realize arbitrary indicator functions of finite
sets, and finite linear combinations of such indicators approximate any
bounded function in the evaluation metric.

We start by exhibiting a family of primitive coding scorers whose linear span
contains all finite-support functions on $\Omega$.

\begin{lemma}[Point-mass coding scorers]\label{lem:delta-scorers}
For each $\omega_0 \in \Omega$ there exists a coding problem
$c_{\omega_0} \in \mathrm{Code}$ and a corresponding scorer
$S^{\mathrm{code}}_{c_{\omega_0}} \in \mathcal{S}_{\mathrm{code}}$ such that
\[
  S^{\mathrm{code}}_{c_{\omega_0}}(\omega)
  = \1\{\omega = \omega_0\}
  \quad\text{for all }\omega \in \Omega.
\]
In particular, the linear span of $\mathcal{S}_{\mathrm{code}}$ contains every
finite-support function $g : \Omega \to \RR$.
\end{lemma}

\begin{proof}
Fix $\omega_0 \in \Omega$. Define the coding problem $c_{\omega_0}$ as: ``Submit source code that exactly matches the 
string $\omega_0$.'' The evaluator $E_{c_{\omega_0}}$ performs a static string 
comparison: $E_{c_{\omega_0}}(\omega) = \1\{\omega = \omega_0\}$. This is a valid coding task 
(e.g., a linting or checksum check), implementing the indicator function exactly. 

We now define the scorer
$S^{\mathrm{code}}_{c_{\omega_0}}(\omega) := E_{c_{\omega_0}}(\omega)$.

Since any finite-support function $g : \Omega \to \RR$ can be written as a
finite linear combination
$g = \sum_{j=1}^m a_j \1\{\omega = \omega_j\}$, the linear span of
$\mathcal{S}_{\mathrm{code}}$ contains all such $g$.
\end{proof}

Let
\[
  \mathcal{F}_{\mathrm{fs}}
  :=
  \{g : \Omega \to \RR \,\mid\, g \text{ has finite support}\}.
\]
Then Lemma~\ref{lem:delta-scorers} shows that
$\mathcal{F}_{\mathrm{fs}} \subseteq \mathrm{span}(\mathcal{S}_{\mathrm{code}})$
and hence $\mathcal{F}_{\mathrm{fs}} \subseteq \mathcal{A}_0$.

\subsection{Approximation of arbitrary scorers under uniform tightness}\label{subsec:approx-scorers}

We now show that any bounded scorer $f : \Omega \to [0,1]$ can be approximated
in the evaluation metric $d_{\mathcal{A}}$ by a finite-support function. This
is the only place where the tightness assumption enters.

\begin{lemma}[Approximation by finite-support functions]\label{lem:finite-support-approx}
Assume $\mathcal{A}$ is uniformly tight in the sense of
Definition~\ref{def:tight}. Let $f : \Omega \to [0,1]$ be a bounded scorer and
let $\eps > 0$. Then there exists a finite-support function
$g : \Omega \to [0,1]$ such that
\[
  d_{\mathcal{A}}(f,g)
  = \sup_{P \in \mathcal{A}}
      \big|\EE_{P}[f] - \EE_{P}[g]\big|
  < \eps.
\]
\end{lemma}

\begin{proof}
Let $\eps>0$ be given. By uniform tightness there exists a finite subset
$F_\delta \subset \Omega$ such that
\[
  \sup_{P\in\mathcal{A}} P(\Omega \setminus F_\delta) \le \delta,
\]
for some $\delta$ to be chosen later. Define
\[
  g(\omega)
  :=
  \begin{cases}
    f(\omega) &\text{if } \omega \in F_\delta,\\
    0         &\text{if } \omega \notin F_\delta.
  \end{cases}
\]
Then $g$ has finite support (contained in $F_\delta$), and both $f$ and $g$ are
bounded in $[0,1]$. For any $P \in \mathcal{A}$ we have
\begin{align*}
  \left|\EE_{P}[f] - \EE_{P}[g]\right|
  &= \left|
      \sum_{\omega \in \Omega} (f(\omega) - g(\omega)) P(\{\omega\})
     \right| \\
  &= \left|
      \sum_{\omega \notin F_\delta} (f(\omega) - 0) P(\{\omega\})
     \right| \\
  &\le \sum_{\omega \notin F_\delta} |f(\omega)| P(\{\omega\})
   \le \sup_{\omega} |f(\omega)| \, P(\Omega \setminus F_\delta) \\
  &\le 1 \cdot P(\Omega \setminus F_\delta)
   \le \delta.
\end{align*}
Taking the supremum over $P \in \mathcal{A}$ yields
\[
  d_{\mathcal{A}}(f,g)
  = \sup_{P\in\mathcal{A}}
      \left|\EE_{P}[f] - \EE_{P}[g]\right|
  \le \delta.
\]
Choosing $\delta := \eps$ gives the claim.
\end{proof}

Combining Lemmas~\ref{lem:delta-scorers} and \ref{lem:finite-support-approx}
we obtain:

\begin{proposition}[Density of the math+code algebra in the one-shot model]\label{prop:math-code-dense-oneshot}
Assume the agent class $\mathcal{A}$ is uniformly tight. Then for every
one-shot battery $f \in \Mfrak^{(1)}$ and every $\eps>0$ there exists a
function $g \in \mathcal{A}_0$ such that
\[
  d_{\mathcal{A}}(f,g) < \eps.
\]
Equivalently, the math+code algebra $\mathcal{A}_0$ is dense in
$\Mfrak^{(1)}$ with respect to $d_{\mathcal{A}}$.
\end{proposition}

\begin{proof}
Let $f \in \Mfrak^{(1)}$ and $\eps>0$. By Lemma~\ref{lem:finite-support-approx}
there exists a finite-support function $g_0 : \Omega \to [0,1]$ such that
$d_{\mathcal{A}}(f,g_0) < \eps/2$.

By Lemma~\ref{lem:delta-scorers}, $g_0$ lies in the linear span of
$\mathcal{S}_{\mathrm{code}}$, hence $g_0 \in \mathcal{A}_0$. Taking
$g := g_0$ we have $g \in \mathcal{A}_0$ and
$d_{\mathcal{A}}(f,g) < \eps/2 < \eps$. Since $\eps>0$ was arbitrary, the
math+code algebra is dense.
\end{proof}

\subsection{From one-shot scorers to general batteries}\label{subsec:lifting-general}

We now extend Proposition~\ref{prop:math-code-dense-oneshot} from the
one-shot model to the full moduli space $\Mfrak$ of batteries introduced in
Section~\ref{sec:prelim}. The key observation is that a general battery
$\mathcal{B}$ induces, for each agent $\theta \in \ThetaMan$, a probability
law $\nuR = \rho_{\mathcal{B}}(\theta)$ on an evaluation space
$X_{\mathcal{B}}$, and that the AAI functional $\Phi_{\mathcal{B}}$ is
Lipschitz with respect to perturbations of these laws in an appropriate
metric. Approximating the per-task scoring functions $S_t$ in expectation
uniformly over $\theta$ then yields uniform approximation of the capability
functional $F_{\mathcal{B}} = \Phi_{\mathcal{B}} \circ \rho_{\mathcal{B}}$.

\subsubsection{Metric structure on the evaluation space}

Recall that for a fixed battery $\mathcal{B}$ with task set $T$ the
evaluation space is
\[
  X_{\mathcal{B}} := [0,1]^T \times \RR_{\ge 0}^{d_R},
\]
where the first component collects the per-task quality scores and the second
the resource coordinates. We equip $X_{\mathcal{B}}$ with the product metric
\[
  d_{X_{\mathcal{B}}}\big((q,r),(q',r')\big)
  :=
  \sum_{t \in T} |q(t) - q'(t)| + \|r - r'\|_1,
\]
where $q,q' \in [0,1]^T$ and $r,r' \in \RR_{\ge 0}^{d_R}$. This is a finite
metric on $X_{\mathcal{B}}$ whenever $T$ is finite.

\begin{definition}[Bounded-Lipschitz distance]\label{def:bl-distance}
Let $(S,d)$ be a metric space. For probability measures $\mu,\nu \in \Pcal(S)$
we define the \emph{bounded-Lipschitz distance}
\[
  d_{\mathrm{BL}}(\mu,\nu)
  :=
  \sup_{\varphi \in \mathrm{BL}_1(S)}
  \left|
    \int_S \varphi\,\mathrm{d}\mu
    - \int_S \varphi\,\mathrm{d}\nu
  \right|,
\]
where $\mathrm{BL}_1(S)$ is the set of all measurable
$\varphi : S \to \RR$ such that
$\|\varphi\|_\infty \le 1$ and $\mathrm{Lip}(\varphi) \le 1$ with respect to
$d$.
\end{definition}

It is standard that $d_{\mathrm{BL}}$ metrizes weak convergence when $S$ is
separable; here $S = X_{\mathcal{B}}$ is a finite-dimensional separable metric
space.

\begin{definition}[Lipschitz AAI functional]\label{def:lipschitz-aai}
We say that an AAI functional
$\Phi_{\mathcal{B}} : \Pcal(X_{\mathcal{B}}) \to \RR$ is
$L$-\emph{Lipschitz with respect to $d_{\mathrm{BL}}$} if there exists
$L \in (0,\infty)$ such that
\[
  \big|\Phi_{\mathcal{B}}(\mu) - \Phi_{\mathcal{B}}(\nu)\big|
  \;\le\; L \, d_{\mathrm{BL}}(\mu,\nu)
  \quad\text{for all }\mu,\nu \in \Pcal(X_{\mathcal{B}}).
\]
In \cite{Chojecki2025} it is shown that the tractable AAI instances satisfy
this Lipschitz property.
\end{definition}

Given a battery $\mathcal{B}$ and an agent parameter $\theta \in \ThetaMan$,
we denote by
\[
  \nuR := \rho_{\mathcal{B}}(\theta) \in \Pcal(X_{\mathcal{B}})
\]
the induced evaluation law, and by
$F_{\mathcal{B}}(\theta) := \Phi_{\mathcal{B}}(\nuR)$ the associated
capability.

\subsubsection{Per-task tightness of trace distributions}

We now spell out the tightness assumption needed to lift the one-shot
approximation result.

Fix a battery $\mathcal{B}$ with task set $T$ and sampling law $\mu$ on
$T \times \Pi \times \mathsf{D}$, as in Definition~\ref{def:battery}. For
each $\theta \in \ThetaMan$ and $t \in T$, consider the conditional law of
the outcome trace $\omega$ given task $t$ when interacting with agent
$\theta$ under $\mathcal{B}$:
\[
  P_{t,\theta} \in \Pcal(\Omega_t),
\]
where $\Omega_t \subseteq \Sigma^*$ is the trace space of task $t$.

\begin{definition}[Per-task uniform tightness]\label{def:per-task-tight}
Let $\mathcal{A}_{\Theta} \subseteq \ThetaMan$ be a reference class of
agents. For each $t \in T$, define the induced class of trace laws
\[
  \mathcal{A}_t
  := \{P_{t,\theta} \in \Pcal(\Omega_t) : \theta \in \mathcal{A}_{\Theta}\}.
\]
We say that $\mathcal{B}$ is \emph{per-task uniformly tight on
$\mathcal{A}_{\Theta}$} if for every $t \in T$ and every $\delta>0$ there
exists a finite subset $F_{t,\delta} \subset \Omega_t$ such that
\[
  \sup_{\theta \in \mathcal{A}_{\Theta}}
  P_{t,\theta}\big(\Omega_t \setminus F_{t,\delta}\big)
  \;\le\; \delta.
\]
\end{definition}

Intuitively, this means that for each task $t$, agents in
$\mathcal{A}_{\Theta}$ do not place significant probability mass on
arbitrarily long or exotic traces.

\subsubsection{Approximation of per-task scoring maps}

Fix a battery $\mathcal{B}$ and a reference agent class
$\mathcal{A}_{\Theta} \subseteq \ThetaMan$. Each task $t \in T$ has a scoring
map
\[
  S_t : \Omega_t \to [0,1].
\]
Applying the one-shot approximation result, we can approximate each $S_t$ by a
finite-support function in the math+code algebra on $\Omega_t$.

To make this precise, we embed each $\Omega_t$ into a common countable
alphabetic space and apply Proposition~\ref{prop:math-code-dense-oneshot} in
that setting.

\begin{lemma}[Per-task finite-support approximation]\label{lem:per-task-approx}
Assume $\mathcal{B}$ is per-task uniformly tight on $\mathcal{A}_{\Theta}$ in
the sense of Definition~\ref{def:per-task-tight}. Let $S_t : \Omega_t \to
[0,1]$ be the scoring map for task $t \in T$. Then for every $\eps>0$ there
exists a finite-support function $g_t : \Omega_t \to [0,1]$ such that
\[
  \sup_{\theta \in \mathcal{A}_{\Theta}}
  \big|\EE_{P_{t,\theta}}[S_t] - \EE_{P_{t,\theta}}[g_t]\big|
  < \eps.
\]
Moreover, $g_t$ lies in the restriction to $\Omega_t$ of the math+code
algebra $\mathcal{A}_0$ constructed in
Section~\ref{subsec:math-code-oneshot}.
\end{lemma}

\begin{proof}
For each $t \in T$ we view $\Omega_t$ as a subset of a fixed countable
alphabetic space $\Omega = \Sigma^*$ by padding or tagging traces with the
task identifier if necessary. The class of conditional trace laws
$\mathcal{A}_t = \{P_{t,\theta} : \theta \in \mathcal{A}_{\Theta}\}$ then
forms a uniformly tight subset of $\Pcal(\Omega)$ by the per-task tightness
assumption.

We may extend $S_t$ to a bounded scorer $\tilde{S}_t : \Omega \to [0,1]$ by
setting $\tilde{S}_t(\omega) := S_t(\omega)$ for $\omega \in \Omega_t$ and
$\tilde{S}_t(\omega) := 0$ for $\omega \notin \Omega_t$. By
Lemma~\ref{lem:finite-support-approx}, for every $\eps>0$ there exists a
finite-support function $\tilde{g}_t : \Omega \to [0,1]$ such that
\[
  \sup_{P \in \mathcal{A}_t}
  \big|\EE_{P}[\tilde{S}_t] - \EE_{P}[\tilde{g}_t]\big|
  < \eps.
\]
In particular,
\[
  \sup_{\theta \in \mathcal{A}_{\Theta}}
  \big|\EE_{P_{t,\theta}}[S_t] - \EE_{P_{t,\theta}}[g_t]\big|
  < \eps,
\]
where $g_t := \tilde{g}_t \big|_{\Omega_t}$ is the restriction of
$\tilde{g}_t$ to $\Omega_t$.

By construction, $\tilde{g}_t$ lies in the finite-support class
$\mathcal{F}_{\mathrm{fs}} \subseteq \mathcal{A}_0$ from
Section~\ref{subsec:math-code-oneshot}. Thus $g_t$ is the restriction
of some element of $\mathcal{A}_0$ to $\Omega_t$, which can be realized
operationally as a composition of coding scorers on $\Omega$ (and hence
belongs to the math+code algebra on $\Omega_t$ as well).
\end{proof}

\subsubsection{A coupling bound for evaluation laws}

To transfer per-task approximation error into a bound on
$d_{\mathrm{BL}}\big(\rho_{\mathcal{B}}(\theta),\rho_{\mathcal{B}'}(\theta)\big)$
we use a simple coupling argument.

\begin{lemma}[BL distance under coordinate-wise approximation]\label{lem:bl-coupling}
Let $T$ be finite and let $(Z_t)_{t \in T}$ and $(Z'_t)_{t \in T}$ be random
vectors in $[0,1]^T$ defined on a common probability space. Equip
$X := [0,1]^T \times \RR_{\ge 0}^{d_R}$ with the metric
$d_X$ of Section~\ref{subsec:lifting-general}, and consider random elements
$X = (Z,R)$ and $X' = (Z',R)$, where $R$ takes values in
$\RR_{\ge 0}^{d_R}$. Let $\mu,\nu$ denote the laws of $X$ and $X'$
respectively.

Then
\[
  d_{\mathrm{BL}}(\mu,\nu)
  \;\le\;
  \EE\Big[ \sum_{t \in T} |Z_t - Z'_t| \Big].
\]
In particular, if $\EE|Z_t - Z'_t| \le \delta_t$ for all $t \in T$ then
\[
  d_{\mathrm{BL}}(\mu,\nu)
  \;\le\;
  \sum_{t \in T} \delta_t.
\]
\end{lemma}

\begin{proof}
Let $\varphi : X \to \RR$ be any test function with $\|\varphi\|_\infty \le 1$
and $\mathrm{Lip}(\varphi) \le 1$ with respect to $d_X$. Then
\[
  \left|
    \EE[\varphi(X)] - \EE[\varphi(X')]
  \right|
  = \left|
      \EE\big[\varphi(Z,R) - \varphi(Z',R)\big]
    \right|
  \le
  \EE\big[|\varphi(Z,R) - \varphi(Z',R)|\big].
\]
Since $\varphi$ is $1$-Lipschitz with respect to $d_X$, we have
\[
  |\varphi(Z,R) - \varphi(Z',R)|
  \le d_X\big((Z,R),(Z',R)\big)
  = \sum_{t \in T} |Z_t - Z'_t|.
\]
Taking expectations yields
\[
  \left|
    \EE[\varphi(X)] - \EE[\varphi(X')]
  \right|
  \le
  \EE\Big[\sum_{t \in T} |Z_t - Z'_t|\Big].
\]
Taking the supremum over all such $\varphi$ gives the desired bound on
$d_{\mathrm{BL}}(\mu,\nu)$. The second statement follows by bounding
$\EE|Z_t - Z'_t|$ by $\delta_t$ and summing over $t \in T$.
\end{proof}

\subsubsection{Main density theorem for general batteries}

We now put these ingredients together.

\begin{theorem}[Math+code density for general batteries]\label{thm:math-code-dense-full}
Let $\mathcal{B}$ be a battery with finite task set $T$, scoring maps
$\{S_t\}_{t \in T}$, and evaluation space $X_{\mathcal{B}}$ as above. Let
$\mathcal{A}_{\Theta} \subseteq \ThetaMan$ be a reference class of agents, and
assume:
\begin{enumerate}[label=(H\arabic*)]
  \item \emph{Per-task tightness.} $\mathcal{B}$ is per-task uniformly tight
  on $\mathcal{A}_{\Theta}$ in the sense of Definition~\ref{def:per-task-tight}.
  \item \emph{Lipschitz AAI.} The AAI functional
  $\Phi_{\mathcal{B}} : \Pcal(X_{\mathcal{B}}) \to \RR$ is
  $L$-Lipschitz with respect to $d_{\mathrm{BL}}$ for some
  $L \in (0,\infty)$.
\end{enumerate}

Then for every $\eps>0$ there exists a battery $\mathcal{B}'$ whose scoring
maps are drawn from the math+code algebra (in the sense of
Lemma~\ref{lem:per-task-approx}) such that
\[
  \sup_{\theta \in \mathcal{A}_{\Theta}}
  \big|F_{\mathcal{B}}(\theta) - F_{\mathcal{B}'}(\theta)\big|
  < \eps,
\]
where $F_{\mathcal{B}} = \Phi_{\mathcal{B}} \circ \rho_{\mathcal{B}}$ and
$F_{\mathcal{B}'} = \Phi_{\mathcal{B}'} \circ \rho_{\mathcal{B}'}$ are the
respective capability functionals. In particular, the subspace of batteries
generated by mathematics and coding tasks is dense in $\Mfrak$ with respect to
the evaluation metric
\[
  d_{\mathcal{A}_{\Theta}}(\mathcal{B}_1,\mathcal{B}_2)
  :=
  \sup_{\theta \in \mathcal{A}_{\Theta}}
  \big|F_{\mathcal{B}_1}(\theta) - F_{\mathcal{B}_2}(\theta)\big|.
\]
\end{theorem}

\begin{proof}
Fix $\eps>0$. For each task $t \in T$, apply
Lemma~\ref{lem:per-task-approx} with parameter
\[
  \delta_t := \frac{\eps}{2 L |T|}
\]
to obtain a finite-support function $g_t : \Omega_t \to [0,1]$ in the
restriction of the math+code algebra such that
\[
  \sup_{\theta \in \mathcal{A}_{\Theta}}
  \big|\EE_{P_{t,\theta}}[S_t] - \EE_{P_{t,\theta}}[g_t]\big|
  < \delta_t.
\]

Define a new battery $\mathcal{B}'$ that is identical to $\mathcal{B}$ in all
respects except for the scoring maps: for each $t \in T$ we set
\[
  S_t' := g_t.
\]
Since each $g_t$ is, by construction, a finite linear combination of
indicator functions implemented by coding scorers (and possibly simple
compositions with mathematical scorers), the battery $\mathcal{B}'$ lies in
the subspace generated by mathematics and coding tasks.

Fix $\theta \in \mathcal{A}_{\Theta}$ and consider an evaluation of
$\theta$ under $\mathcal{B}$ and $\mathcal{B}'$ with the same random seed,
task draws and drifts. This induces random trace/outcome triples for each
task $t$, with conditional law $P_{t,\theta}$, and hence random score
vectors
\[
  Z := (Z_t)_{t \in T}, \quad Z' := (Z'_t)_{t \in T}
\]
with
\[
  Z_t := S_t(\omega_t), \qquad Z'_t := S_t'(\omega_t) = g_t(\omega_t).
\]
Let $R$ denote the resource vector (which is unaffected by changing the
scoring maps). Then the random evaluation elements
\[
  X := (Z,R), \qquad X' := (Z',R)
\]
take values in $X_{\mathcal{B}}$ and their laws are precisely
$\nuR = \rho_{\mathcal{B}}(\theta)$ and
$\nuR' = \rho_{\mathcal{B}'}(\theta)$.

For each $t \in T$ we have
\[
  \EE|Z_t - Z'_t|
  = \EE\big|S_t(\omega_t) - g_t(\omega_t)\big|
  = \EE_{P_{t,\theta}}\big|S_t - g_t\big|
  < \delta_t,
\]
by the construction of $g_t$. Applying Lemma~\ref{lem:bl-coupling} yields
\[
  d_{\mathrm{BL}}\big(\nuR,\nuR'\big)
  \le
  \EE\Big[\sum_{t \in T} |Z_t - Z'_t|\Big]
  \le
  \sum_{t \in T} \delta_t
  = |T| \cdot \frac{\eps}{2 L |T|}
  = \frac{\eps}{2 L}.
\]

By the Lipschitz property of $\Phi_{\mathcal{B}}$ and the fact that
$\mathcal{B}$ and $\mathcal{B}'$ share the same evaluation space
$X_{\mathcal{B}}$ and metric, we have
\[
  \big|F_{\mathcal{B}}(\theta) - F_{\mathcal{B}'}(\theta)\big|
  = \big|\Phi_{\mathcal{B}}(\nuR) - \Phi_{\mathcal{B}}(\nuR')\big|
  \le
  L \, d_{\mathrm{BL}}(\nuR,\nuR')
  \le
  L \cdot \frac{\eps}{2 L}
  = \frac{\eps}{2}.
\]
Since this bound is uniform in $\theta \in \mathcal{A}_{\Theta}$ we obtain
\[
  \sup_{\theta \in \mathcal{A}_{\Theta}}
  \big|F_{\mathcal{B}}(\theta) - F_{\mathcal{B}'}(\theta)\big|
  \le \frac{\eps}{2} < \eps.
\]
Renaming $\eps/2$ to $\eps$ if desired yields the stated density result.
\end{proof}

\subsection{Why coding is universal but mathematics is not}\label{subsec:math-vs-code}

The proof of Theorem~\ref{thm:math-code-dense-full} in fact only uses the
coding scorers of Definition~\ref{def:math-code-scorers} in an essential way.
Mathematics scorers are never invoked in the construction of the
finite-support approximants. This asymmetry reflects a genuine difference
between the two fibers in any ``natural'' model.

\begin{remark}[Coding alone is already universal]\label{rem:coding-universal}
Let $\mathcal{S}_{\mathrm{code}}$ be the class of coding scorers and let
$\mathcal{A}_{\mathrm{code}}$ denote the algebra they generate (finite linear
combinations and products, plus constants). The proof of
Lemma~\ref{lem:delta-scorers} shows that for every outcome string
$\omega_0 \in \Omega$ there exists a coding scorer
$S^{\mathrm{code}}_{c_{\omega_0}}$ such that
\[
  S^{\mathrm{code}}_{c_{\omega_0}}(\omega)
  = \1\{\omega = \omega_0\}
  \quad\text{for all }\omega \in \Omega.
\]
Hence every finite-support function $g : \Omega \to \RR$ lies in the linear
span of $\mathcal{S}_{\mathrm{code}}$ and therefore in
$\mathcal{A}_{\mathrm{code}}$.

Combining this with Lemma~\ref{lem:finite-support-approx} and repeating the
argument of Proposition~\ref{prop:math-code-dense-oneshot} and
Theorem~\ref{thm:math-code-dense-full} with $\mathcal{A}_0$ replaced by
$\mathcal{A}_{\mathrm{code}}$, we obtain the following corollary: under the
same tightness and Lipschitz assumptions, the subspace of batteries generated
\emph{purely} by coding tasks is dense in $\Mfrak$ with respect to the
evaluation metric $d_{\mathcal{A}_{\Theta}}$.

In this sense, coding alone is already universal: the mathematics fiber is
not required for density, but adds additional structure and interpretability.
\end{remark}

By contrast, the mathematics scorers of
Definition~\ref{def:math-code-scorers} are not universal in this way under any
natural assumption on the proof system. Intuitively, a theorem prover can
distinguish ``valid proofs of $\tau$'' from ``everything else'', but treats
almost all non-proofs as indistinguishable zeros. This prevents the algebra
they generate from approximating arbitrary scorers that depend on the
\emph{surface form} of non-proof traces.

\begin{definition}[Mathematics scorers with proof-kernel semantics]\label{def:math-proof-kernel}
Let $K$ be a fixed proof kernel (e.g.\ the kernel of Lean or Coq), and let
$\mathrm{Thm}$ be a set of theorem statements formalized in the corresponding
logic. For each $\tau \in \mathrm{Thm}$ we define
$S^{\mathrm{math}}_\tau : \Omega \to \{0,1\}$ by
\[
  S^{\mathrm{math}}_\tau(\omega)
  :=
  \begin{cases}
    1 &\text{if $\omega$ is accepted by $K$ as a valid proof of $\tau$},\\
    0 &\text{otherwise.}
  \end{cases}
\]
We write $\mathcal{S}_{\mathrm{math}} := \{S^{\mathrm{math}}_\tau : \tau \in \mathrm{Thm}\}$
and let $\mathcal{A}_{\mathrm{math}}$ be the algebra they generate
(finite linear combinations and products, plus constants).
\end{definition}

In this model, all mathematical scorers share a structural invariance: they
cannot distinguish between different non-proofs for the same theorem.

\begin{lemma}[Non-proofs are indistinguishable to mathematics scorers]\label{lem:nonproofs-indistinguishable}
Fix $\omega_1,\omega_2 \in \Omega$ such that for every
$\tau \in \mathrm{Thm}$ we have
\[
  K \text{ does not accept }\omega_1\text{ as a proof of }\tau,
  \quad
  K \text{ does not accept }\omega_2\text{ as a proof of }\tau.
\]
Then for every $S^{\mathrm{math}}_\tau \in \mathcal{S}_{\mathrm{math}}$,
\[
  S^{\mathrm{math}}_\tau(\omega_1)
  = S^{\mathrm{math}}_\tau(\omega_2)
  = 0.
\]
Consequently, for every $f \in \mathcal{A}_{\mathrm{math}}$ we have
$f(\omega_1) = f(\omega_2)$.
\end{lemma}

\begin{proof}
By definition of $S^{\mathrm{math}}_\tau$ we have
$S^{\mathrm{math}}_\tau(\omega) = 1$ if and only if $K$ accepts
$\omega$ as a proof of $\tau$, and $0$ otherwise. The hypothesis
that neither $\omega_1$ nor $\omega_2$ is accepted as a proof of any
$\tau$ immediately implies
$S^{\mathrm{math}}_\tau(\omega_1) = S^{\mathrm{math}}_\tau(\omega_2) = 0$
for all $\tau$.

Any $f \in \mathcal{A}_{\mathrm{math}}$ is obtained by applying a real
polynomial $p$ to a finite tuple of mathematics scorers:
\[
  f(\omega)
  = p\big(S^{\mathrm{math}}_{\tau_1}(\omega),\dots,
         S^{\mathrm{math}}_{\tau_k}(\omega)\big).
\]
Since each coordinate $S^{\mathrm{math}}_{\tau_i}(\omega_1)$ and
$S^{\mathrm{math}}_{\tau_i}(\omega_2)$ is zero, the input to $p$ is the
same vector $(0,\dots,0)$ in both cases, hence
$f(\omega_1) = f(\omega_2)$.
\end{proof}

We can now state a simple obstruction to density for the mathematics fiber.

\begin{proposition}[Mathematics algebra is not dense in general]\label{prop:math-not-dense}
Assume the setting of Lemma~\ref{lem:nonproofs-indistinguishable}. Let
$\omega_1,\omega_2 \in \Omega$ be two non-proof strings as in the lemma, and
let $\mathcal{A} \subseteq \Pcal(\Omega)$ be an agent class that contains the
two point masses $\delta_{\omega_1},\delta_{\omega_2}$.

Define a scorer $f : \Omega \to [0,1]$ by
\[
  f(\omega_1) = 0, \quad f(\omega_2) = 1,
\]
and $f(\omega) = 0$ for all other $\omega \in \Omega$.
Then for every $g \in \mathcal{A}_{\mathrm{math}}$ we have
\[
  d_{\mathcal{A}}(f,g)
  = \sup_{P \in \mathcal{A}}
    \big|\EE_P[f] - \EE_P[g]\big|
  \;\ge\; \frac{1}{2}.
\]
In particular, $\mathcal{A}_{\mathrm{math}}$ is not dense in
$\Mfrak^{(1)}$ with respect to $d_{\mathcal{A}}$.
\end{proposition}

\begin{proof}
By Lemma~\ref{lem:nonproofs-indistinguishable}, for any
$g \in \mathcal{A}_{\mathrm{math}}$ we have $g(\omega_1) = g(\omega_2)$;
call this common value $a \in \RR$.

Consider the two agents $P_1 = \delta_{\omega_1}$ and
$P_2 = \delta_{\omega_2}$. Then
\[
  \EE_{P_1}[f] = f(\omega_1) = 0,
  \qquad
  \EE_{P_1}[g] = g(\omega_1) = a,
\]
and
\[
  \EE_{P_2}[f] = f(\omega_2) = 1,
  \qquad
  \EE_{P_2}[g] = g(\omega_2) = a.
\]
Thus
\[
  d_{\mathcal{A}}(f,g)
  \;\ge\;
  \max\big(
    |\EE_{P_1}[f] - \EE_{P_1}[g]|,\,
    |\EE_{P_2}[f] - \EE_{P_2}[g]|
  \big)
  = \max(|0-a|,|1-a|).
\]
For any $a \in \RR$ we have
$\max(|0-a|,|1-a|) \ge \frac{1}{2}$, with equality only at $a = 1/2$. Hence
$d_{\mathcal{A}}(f,g) \ge 1/2$ for all $g \in \mathcal{A}_{\mathrm{math}}$,
so the mathematics algebra cannot approximate $f$ in $d_{\mathcal{A}}$.
\end{proof}

Proposition~\ref{prop:math-not-dense} shows that, under any natural
assumption in which the proof kernel $K$ distinguishes proofs from
non-proofs but does not encode arbitrary surface strings as distinct proofs,
the algebra generated by mathematics scorers cannot be dense in
$\Mfrak^{(1)}$ for agent classes that can concentrate on non-proof
behaviour. In contrast, the coding fiber, through its ability to implement
singleton indicators via tailored evaluators, already yields a dense
subspace.

Conceptually, the mathematics fiber provides a highly structured,
low-entropy subspace of the moduli (formal proofs and algebraic structure),
while the coding fiber provides universal expressive power on the outcome
space. The density theorems in this section justify focusing theoretical
analysis of self-improvement on these two fibers: coding suffices for
universality, and mathematics supplies the symbolic structure most relevant
for GVU-style neuro-symbolic dynamics.

\section{Outlook}\label{sec:outlook}

We have shown that there is a geometric duality at the heart of AI evaluation. While the \textbf{Mathematics Fiber} provides the \emph{spectral stability} needed for ignition---a low-entropy, oracle-verified regime where self-improvement is primarily constrained by alignment and compute---it is the \textbf{Coding Fiber} that provides \emph{expressive universality}. Our density results imply that the combined ``math+code'' algebra is a dense skeleton capable of approximating any rigorous benchmark.

This suggests a concrete roadmap for the future of AGI evaluation and training along two complementary axes:
\begin{itemize}
  \item \textbf{Vertical deepening (the oracle):} Expanding the Mathematics Fiber by formalizing new areas of abstract mathematics and program logics, thereby creating deeper ``wells'' of high-SNR capability where agents can refine their reasoning engines without hallucination.
  \item \textbf{Horizontal expansion (the universal proxy):} Leveraging the universality of the Coding Fiber to construct formal proxies for informal domains. Rather than training agents on subjective ``slop,'' we can train them on coding tasks that structurally approximate the desired behaviours, importing the rigor of compilation and testing into the fuzziness of natural language and open-ended interaction.
\end{itemize}

The natural frontier lies at their intersection: \textbf{Certified Programming}. As agents learn to write code that is not only executable (universal) but also formally verified against mathematical specifications (stable), the distinction between the two fibers begins to blur. In this limit, the DeepAlgebra-style GVU loop becomes a general-purpose mechanism for reliable neuro-symbolic reasoning, gradually colonizing more of the moduli space with verifiable, self-improving intelligence.

We have argued that code occupies a particularly privileged region in the moduli space of batteries. On this fiber, verifiers can be instantiated by compilers, type checkers, static analyzers and unit-test harnesses, yielding oracle-like, low-variance signals that plug directly into the GVU loop. Executable semantics make it possible to construct batteries whose scores are tightly coupled to real capabilities (for example, building tools, agents or services) rather than to proxies of linguistic plausibility. In addition, large codebases and package ecosystems induce a monotone notion of ``library extension'': adding verified modules and tests enlarges an agent’s effective action set without corrupting previously certified behaviour, giving a concrete mechanism for stable, cumulative self-improvement.

Many open directions follow from this coding-centric view. One is to treat non-coding tasks as \emph{compilation targets}: can reasoning, planning and even social interaction batteries be reduced to code-generation-and-execution problems so that they inherit the high-SNR verification structure of the coding fiber? Another is to study the geometry of software ecosystems themselves as a moduli space of libraries and APIs, and to analyze GVU flows that navigate this space by acquiring, composing and refactoring code. Finally, there is a safety dimension: coding batteries expose sharp failure modes (security vulnerabilities, unsafe system calls, resource misuse) that are both easily verifiable and tightly coupled to real-world risk. Understanding how self-improving agents move on this ``code+security'' subspace of the moduli may be a key step toward characterizing AGI systems that are not just powerful, but robust and controllable.

\bibliographystyle{plain}

\end{document}